\pdfoutput=1
\documentclass{article}
\PassOptionsToPackage{numbers, compress}{natbib}

\usepackage[preprint]{neurips_2023}

\usepackage{url,color,cite}
\usepackage{caption}
\usepackage{amsmath,amsthm,amsfonts,amssymb}
\usepackage{mathtools}

\usepackage{floatrow}
\usepackage{epstopdf}
\usepackage{nicefrac}
\usepackage[noend]{algorithmic}
\usepackage{algorithm}
\usepackage{pifont}
\usepackage{subcaption}
\usepackage{sidecap}
\usepackage{multirow}
\usepackage{multicol}
\usepackage{bbm}
\usepackage{enumitem}
\usepackage{sidecap}

\newtheorem{lemma}{Lemma}
\newtheorem*{lemma*}{Lemma}
\newtheorem{theorem}{Theorem}
\newtheorem*{theorem*}{Theorem}
\theoremstyle{definition}
\newtheorem{definition}{Definition}
\newtheorem*{definition*}{Definition}
\theoremstyle{definition}

\newtheorem{remark}{Remark}
\theoremstyle{definition}
\newtheorem{corollary}{Corollary}
\newtheorem*{corollary*}{Corollary}
\theoremstyle{definition}

\newtheorem*{claim*}{Claim}

\newcommand{\bbR}{\mathbb{R}}

\newcommand{\calA}{\mathcal{A}}

\newcommand{\calD}{\mathcal{D}}

\newcommand{\calM}{\mathcal{M}}

\newcommand{\calS}{\mathcal{S}}

\newcommand{\calX}{\mathcal{X}}
\newcommand{\calY}{\mathcal{Y}}

\renewcommand{\(}{\left(}
\renewcommand{\)}{\right)}

\newcommand{\eps}{\epsilon}

\newcommand{\SUB}[1]{\hspace{-0.15in} \textbf{#1}}
\newcommand{\mycaptionof}[2]{\captionof{#1}{#2}}

\usepackage{microtype}
\usepackage{graphicx}
\usepackage{booktabs} 
\usepackage{hyperref}
\usepackage{cleveref}

\usepackage{lipsum}

\usepackage[utf8]{inputenc} 
\usepackage[T1]{fontenc}    
\usepackage{hyperref}       
\usepackage{url}            
\usepackage{booktabs}       
\usepackage{nicefrac}       
\usepackage{microtype}      
\usepackage{xcolor}         
\usepackage{wrapfig}

\newcommand*\samethanks[1][\value{footnote}]{\footnotemark[#1]}

\author{Seng Pei Liew \thanks{LINE Corporation. \texttt{\{sengpei.liew, satoshi.hasegawa, tsubasa.takahashi\}@linecorp.com}}
\And
Satoshi Hasegawa \samethanks[1]
\And
Tsubasa Takahashi \samethanks[1]}

\date{}

\begin{document}
\title{Privacy Amplification via Shuffled Check-Ins}

\maketitle

\begin{abstract}

We study a protocol for distributed computation called shuffled check-in, which achieves strong privacy guarantees without requiring any further trust assumptions beyond a trusted shuffler.
Unlike most existing work, shuffled check-in allows clients to make independent and random decisions to participate in the computation, removing the need for server-initiated subsampling. 
Leveraging differential privacy, we show that shuffled check-in achieves tight privacy guarantees through privacy amplification, with a novel analysis based on R{\'e}nyi differential privacy that improves privacy accounting over existing work. 
We also introduce a numerical approach to track the privacy of generic shuffling mechanisms, including Gaussian mechanism, which is the first evaluation of a generic mechanism under the distributed setting within the local/shuffle model in the literature. 
Empirical studies are also given to demonstrate the efficacy of the proposed approach.
\end{abstract}



\section{Introduction}\label{sec:introduction}


Distributed computation, particularly cross-device federated learning (FL), is a highly scalable and privacy-friendly computational framework of large-scale data analysis and machine learning.
To achieve rigorous privacy guarantees, FL has been combined with differential privacy (DP) \citep{dwork2006calibrating,dwork2006our} in several efforts \citep{mcmahan2017learning, kairouz2021advances}.
This is however typically studied under central DP, a trust model requiring rather strong trust assumptions.
Under central DP, individuals have to entrust the server to handle complex operations which increases the risk of privacy breaches, making it a less desirable trust assumption in practical deployment.

The local DP (LDP) model \citep{evfimievski2003limiting,kasiviswanathan2011can}, on the other hand, requires minimal trust assumption and is preferred in industry for practical deployment \citep{erlingsson2014rappor,apple2017learning,ding2017collecting}.
LDP allows users to randomize their report without requiring any trusted third-party entity, which significantly reduces the risk of privacy breaches.
However, this model comes at the expense of significant utility (e.g., classification accuracy of neural network) loss.
Therefore, intermediate models between the local and central models that balance the trust/utility trade-off has garnered attention in both academia and industry.

The shuffle/shuffled model is perhaps the most suitable intermediate model in practice \citep{bittau2017prochlo,cheu2019distributed,erlingsson2019amplification}.
The essence of this model is data anonymization: instead of being given the access of each user's LDP report, the adversary/analyzer is exposed only to a set of \textit{anonymized} LDP reports from the shuffler.
``Privacy amplification" via shuffling is shown to be feasible, where one can achieve the same utility with a smaller DP budget.

While the studies of novel architecture/techniques in FL with DP are gaining traction, they often require rather sophisticated trust assumptions beyond those achievable by well-established security/trust models (e.g., shuffler or the secure aggregation protocol \citep{balle2020privacy, kairouz2021practical, nguyen2022federated}).
For example, \citep{nguyen2022federated} assumes that the server is able to collect, clip, randomize and aggregate user update without violating privacy.
\textit{By contrast, our focus is on achieving rigorous privacy guarantees of distributed computation making as few trust assumptions as possible}: we pursue a ``minimalist" shuffle-model approach (in the sense that we utilize only the shuffle model as the trust model; the ultimate minimalist approach without shuffler would be the local DP model) and make full use of its privacy amplficiation effects, instead of making more sophisticated trust assumption.
We emphasize that this is particularly relevant in industrial deployment where user privacy is at stake and is of priority: simple, minimal and well-established models are preferable in gaining user consent/trust.

With this in mind, our research question is: \textit{how can we leverage the shuffle model to achieve tight privacy accounting for distributed computation, without relying on further trust assumption?} 

In this paper, we study shuffled check-in, a minimal framework utilizing only the shuffler.
Our protocol achieves this by letting user to randomize her report, and ``check in" to the shuffler randomly and autonomously, without putting trust on the server to index the users explicitly.
This allows for privacy amplification via shuffling and subsampling \citep{ullman}, where the randomness of user participation in learning further enhances the privacy guarantee.
See Figure \ref{fig:sci_scheme} for an illustration.

We note that \citep{girgis2021differentially} has proposed a protocol called self sampling, which is essentially the same as our shuffled check-in protocol.
However, there, the analysis on the varying number of user is done by "fixing" it using a high-probability bound, strong composition is used, and only the pure DP mechanism and its theoretical bound is studied.
Here, we present a significant and non-trivial extension of previous work: we achieve a \textit{stark improvement in privacy guarantees} (using RDP without employing such a high-probability bound) and \textit{generalization} (covering approximate DP mechanism especially Gaussian mechanism) of the protocol as well as providing \textit{concrete empirical evaluations}.

\begin{figure}
\begin{floatrow}
\ffigbox{%
    \includegraphics[width=0.38\textwidth]{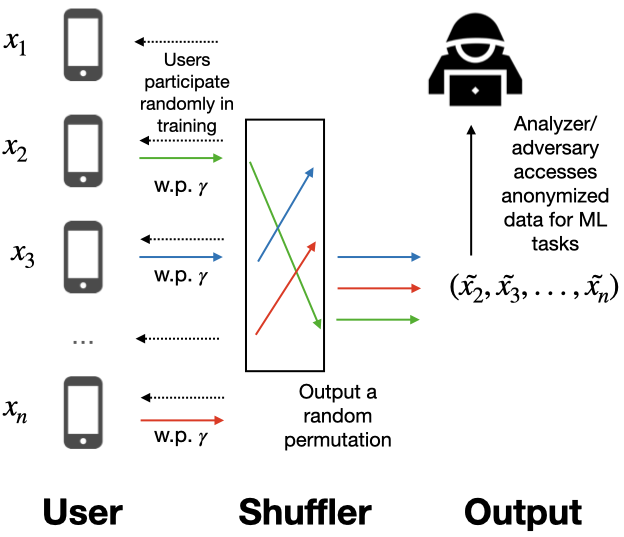}
}{%
  \caption{Shuffled check-in protocol works as follows: In each round, user randomly decides whether to participate in the training with probability $\gamma$, without being indexed explicitly by the server.
  Reports contributed by the users are shuffled/anonymized before being forwarded to the analyzer.
  Shuffled check-in hence requires no trust assumption beyond the shuffler.
  }
\label{fig:sci_scheme}
}%
\ffigbox{%
    \includegraphics[width=0.48\textwidth]{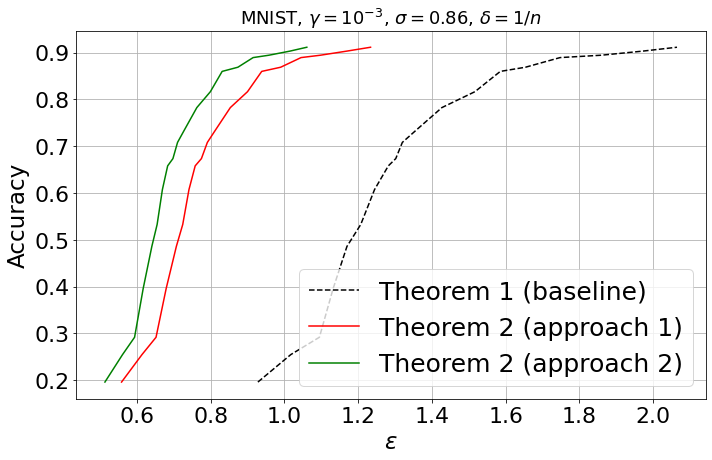}
}{%
  \caption{Our tighter privacy analysis achieves high classification accuracy at lower privacy budget on the MNIST dataset in the distributed setting (details in Section \ref{sec:discussion}). Our approaches 1 and 2 achieves stronger privacy amplification via Equation \ref{eq:main} and outperform the naive baseline (Theorem 1) based on \citep{girgis2021differentially}.
  \label{fig:exp}
    }%
}
\end{floatrow}
\end{figure}
\subsection{Our contribution}
To this end, we analyze the R{\'e}nyi differential privacy (RDP) of a novel subsampling mechanism to characterize shuffled check-in's privacy amplification effects, culminating to Theorem \ref{th:main}, encapsulated by the following expression:
\begin{equation*}
    \eps(\lambda) \leq \frac{1}{\lambda-1} \log \left(
    \sum_{k=0}^{n}  \binom{n}{k} \gamma^{k} (1-\gamma)^{n-k} \mathbb{E}_{q_k}\left[\left(\frac{p_k}{q_k}\right)^{\lambda}\right]
    \right).
\end{equation*}
See Section \ref{subsec:sci_rdp} for details.
We show that our techniques provide much tighter privacy accounting.

Another contribution we make in this paper is the application of Gaussian mechanism within the context of shuffled check-in (and shuffling). 
The Gaussian mechanism is optimal for DP empirical risk minimization \citep{bassily2014private} and easy to implement compared to alternative methods,
\footnote{For example, the LDP-SGD algorithm \citep{duchi2018minimax,erlingsson2020encode}, which proceeds roughly as follows. 
The client-side algorithm first clips user gradient, and flips its sign with a certain probability.
Then, a unit vector is sampled randomly to make an inner product with the processed gradient to yield the inner product's sign.
The sign is again flipped with a certain probability, before sending it along with the unit vector to the server.
The server-side algorithm includes normalizing the aggregated reports, and projecting them to a convex set before updating the model.
See, e.g., \citep{erlingsson2020encode} for details.}
but it has not been studied before possibly due to the difficulty of existing techniques in handling approximate DP mechanisms like the Gaussian mechanism.
We propose a numerical but practical approach to track the privacy budget for such a mechanism.
A lower bound for Gaussian mechanism is also derived to characterize the tightness of the upper bound.
A snapshot of our results is shown in Figure \ref{fig:exp}, where it can be seen that our approaches provide tighter (lower $\epsilon$) privacy guarantees compared to baseline. 


The novelty and highlight of this work are outlined as follows.
\begin{itemize}[leftmargin=*]
\item 
\textit{The shuffled check-in mechanism we study is a new and novel type of subsampling.}
In particular, we focus on cases where the base mechanism (DP mechanism before subsampling) depends on the number of subsampled users, which we refer to as ``population-aware" (see Section \ref{subsec:related}).
This consideration has not been explored in the existing literature, which only focuses on base mechanisms that are agnostic of the number of subsampled users.
Therefore, we develop new tools to analyze this type of mechanism in Section \ref{sec:sci}. Our study of this new form of subsampling is motivated by the need for minimal trust, such as in the shuffled check-in protocol. 
Additionally, population-aware base mechanisms are applicable to secure aggregation protocols or distributed DP mechanisms without explicit user indexing (Section \ref{sec:discussion}).
Consequently, our findings have practical significance.

\item 
\textit{Gaussian mechanism in the context of shuffling:} 
The Gaussian mechanism is one of the most fundamental and commonly used DP mechanisms, yet its RDP composition property has not been examined under the context of shuffling.
In Sections  \ref{sec:sci} and \ref{sec:lower}, we introduce new analytical and numerical tools to analyze the Gaussian mechanism in the context of shuffling.
Specifically, we provide upper and lower bounds that characterize its privacy properties.
 \item 
 \textit{Impacts:} 
 Our findings have direct implications for improving and generalizing the shuffled check-in protocol analyzed using existing techniques.
 We further emphasize that the impact of our techniques extends beyond this protocol. 
 Our methods can not only be applied to the distributed DP scenario, as mentioned earlier, but also provide insights into centralized learning, which typically involves conducting DP-SGD through shuffled mini-batches while often ignoring the shuffling effect. 
 We discuss the wider implications of our research in Section \ref{sec:discussion}.
\end{itemize}

\subsection{Related work in privacy accounting}
\label{subsec:related}
Here, we mainly discuss related works in privacy accounting and their difference compared to our subsampling scheme.
Other related work can be found in Appendix \ref{app:pre}.
Typically in FL, a subsampling procedure is taken, where in each round of training, only clients sampled in a uniform and random manner by the server participate in the training.
This randomness of subsampling leads to privacy amplification in FL with DP, critical at achieving acceptable levels of utility under meaningful DP guarantees \citep{abadi2016deep,bassily2014private,kasiviswanathan2011can}.
Discussing privacy amplification via subsampling is intricate: one has to be careful about what definition of neighboring dataset in use is (removal or replacement), and the type of subsampling (subsampling with/without replacement, Poisson subsampling) employed \citep{balle2018privacy,wang2019subsampled,steinke2022composition}.
That each user has certain probability of checking in as in our protocol is similar to Poisson subsampling, where the number of user or data instance checking in at each round is a variable. 
However, there is a key difference: in previous works, the noisy output (in the simplest case, a scalar value) \textit{does not} leak information about the number of check-ins, whereas our protocol outputs a set of values where the cardinality of the set is the number of check-ins.
More technically, the base mechanism (shuffling) in our protocol is dependent on the number of check-ins, or in other words coupled to the subsampling operation (\textit{population-aware}), whereas previous works deal only with base mechanism that is independent of the number of user/data instance (decoupled from the subsampling operation \citep{balle2018privacy,zhu2019poission,girgis2021renyia}).
This leads to new non-trivial analysis to be presented in Section \ref{sec:sci}. 
We note particularly that the Poisson subsampled shuffling analysis studied in \citep{girgis2021renyia} is flawed as the facts that the shuffling base mechanism is dependent on the number of check-ins, and that shuffling should be considered under replacement DP, have been overlooked.
We give a proper treatment including these considerations in this work.

\newlength{\oldintextsep}
\setlength{\oldintextsep}{\intextsep}

\setlength\intextsep{0pt}

\section{Problem Setup and Preliminaries}\label{sec:pre}
\noindent\textbf{Problem setup.}
We consider the distributed setting, where there are $n$ users each holding a report $x_i$ for $i \in [n]$, and the whole decentralized set of reports is denoted by $D = (x_1,\dotsc,x_n)$.
As mentioned earlier, we assume the existence of a trustworthy shuffler that is able to establish a secure communication channel with each user to receive the (randomized) report.
The untrusted server/aggregator can communicate with all users (assumed to be honest but curious) but cannot index any specific user nor receive report directly from user.

We next present important definitions and known results of differential privacy.
\begin{definition}[Central Differential Privacy \citep{Dwork14}]
Given $\epsilon \geq 0$ and $\delta \in [0,1]$, a randomization mechanism, $\mathcal{M}: \mathcal{D}^n \rightarrow \calS$ with domain $\mathcal{D}^n$ and range $\mathcal{S}$ satisfies central ($\epsilon$, $\delta$)-differential privacy (DP) if for any two adjacent databases $D, D' \in \mathcal{D}^n$ with $n$ data instances and for any subset of outputs $S \subseteq \mathcal{S}$, the following holds:
\begin{align}
\Pr[\mathcal{M}(D) \in S] \leq e^\epsilon \cdot \Pr[\mathcal{M}(D') \in S] + \delta.
\end{align}
\end{definition}
We say that an $(\eps,\delta)$-DP mechanism satisfies \textit{approximate} DP, or is $(\eps,\delta)$-indistinguishable.
$(\eps,\delta)$-DP is also simply referred to as ``DP" when the context is clear.
Moreover, we mostly work with the "replacement" version of DP, where adjacent databases have one data instance replaced by another data instance. \footnote{This is the natural definition for shuffling: the two adjacent datasets should have the same number of shuffled data instances.}
When $D$ consists of only a single element, the mechanism, also known as a local randomizer, is said to be satisfying local DP:
\begin{definition}[Local Differential Privacy (LDP) \citep{kasiviswanathan2011can}]
A randomization mechanism $\calA: \mathcal{D}^{n} \rightarrow \mathcal{S}$ satisfies local $(\eps,\delta)$-DP if for all pairs $x,x'\in \mathcal{D}$, $\calA(x)$ and $\calA(x')$ are $(\eps,\delta)$-indistinguishable.
\end{definition}
We often refer to a mechanism satisfying $(\eps,0)$ LDP as an $\eps$-LDP randomizer.
Next, we introduce the notion of shuffling.
\begin{definition}[Shuffling \citep{cheu2019distributed,erlingsson2019amplification}]
For any domain $\mathcal{D}$, let $\mathcal{A}^{i}:\mathcal{S}^{1}\times\cdots\times\mathcal{S}^{i-1}\times\mathcal{D}\to\mathcal{S}^{i}$ for $i\in[n]$ (where $\mathcal{S}^{i}$ is the range space of $\mathcal{A}^{i}$) be a sequence of algorithms such that $\mathcal{A}^i(z_{1:i-1}, \cdot)$ is a LDP randomizer for all values of auxiliary inputs for $z_{1:i-1}\in\mathcal{S}^{1}\times\cdots\times\mathcal{S}^{i-1}$. 
 $\mathcal{A}_{s}:\mathcal{D}^{n}\to\mathcal{S}^{1}\times\cdots\times \mathcal{S}^{n}$ is the shuffling algorithm that given a dataset, $x_{1:n}\in\mathcal{D}^{n}$, samples a uniform random permutation $\pi$ over $[n]$, then sequentially computes $z_i=\mathcal{R}[i](z_{1:i-1}, x_{\pi(i)})$ for $i\in[n]$ and outputs $z_{1:n}$. 
\end{definition}
For $\mathcal{A}^{i}$ that satisfies $(\epsilon_0,\delta_0)$-DP, if $\mathcal{A}_{s}$ satisfies $(\epsilon,\delta)$ such that $\epsilon \leq \epsilon_0$, we say that privacy amplification has occured via shuffling.  
We next introduce R{\'e}nyi differential privacy \citep{mironov2017renyi,dwork2016concentrated,bun2016concentrated}, the main privacy notion used in this paper.
\begin{definition}[R{\'e}nyi Differential Privacy (RDP)~\citep{mironov2017renyi}]\label{def:RDP}
A randomization mechanism $\mathcal{M}: \mathcal{D}^{n} \rightarrow \mathcal{S}$ is $\epsilon$-R{\'e}nyi differential privacy of order $\lambda\in(1,\infty)$ (or $(\lambda,\epsilon)$-RDP), if for any adjacent databases $D$, $D'\in\calD^{n}$, the R{\'e}nyi divergence of order $\lambda$ between $\calM(D)$ and $\calM(D')$ is upper-bounded by $\eps$:
$D_{\lambda}(\calM(D)||\calM(D'))=\frac{1}{\lambda-1}\log\left(\mathbb{E}_{\phi\sim\calM(D')}\left[\left(\frac{\calM(D)(\phi)}{\calM(D')(\phi)}\right)^{\lambda}\right]\right) \nonumber \leq \epsilon$,
where $\calM(\calD)(\phi)$ denotes the probability of $\calM$ which takes $D$ as input outputting $\phi$.
\end{definition}
We sometimes write $\eps$ as $\eps(\lambda)$ to indicate that it is a function of $\lambda$.
Given an approximate DP mechanism, we may wish to convert it to the RDP to obtain tighter composition-based privacy accounting. 
The conversion is performed with the following lemma.
\begin{lemma}[DP-to-RDP conversion \citep{asoodeh2021three}]
\label{lm:dprdp}
If a mechanism $\calM$ is $\left(\epsilon,\delta\right)$-DP, $\calM$ is also $\left(\lambda,\epsilon\left(\lambda\right)\right)$-RDP, where $\lambda > 1$ and $\eps(\lambda)$ is given by
\begin{equation}
\label{eq:dprdp}
\epsilon(\lambda) = \epsilon + \frac{1}{\lambda -1}\log M(\lambda,\epsilon,\delta),
\end{equation}
where
$M(\lambda,\epsilon,\delta) = \min_{r \in (\delta, 1)} \left( r^{\lambda}(r-\delta)^{1-\lambda} + (1-r)^{\lambda} (e^{\eps}-r+\delta)^{1-\lambda} \right).$
\end{lemma}
Note that $M(\lambda,\epsilon,\delta)$ can be solved efficiently as it is a convex optimization problem \citep{asoodeh2021three}.
We relegate other preliminaries to Appendix \ref{app:pre}.
\section{Shuffled Check-in}
\label{sec:sci}
In this section, we first give a formal description of the protocol, then we present the first-cut approximate DP-based privacy analysis of the protocol, followed by our core RDP-based analysis.
All proofs of our theorems can be found in Appendix \ref{app:proof}.

\noindent \textbf{The protocol.}
We begin by describing the protocol of shuffled check-in.
In each round $t$, a message is broadcast to all clients to ask for participation in the computation.
Each client flips a biased coin to decide whether to participate.
The probability of a client participating in the computation is modeled by a parameter $\gamma$ ($0 \leq \gamma \leq 1$), which we call the \textit{check-in rate}.
The participating probability follows the Bernoulli distribution, Bern($\gamma$).

The protocol is also described in Algorithm \ref{alg:protocol} in Appendix \ref{app:algo}.
Writing symbolically ``shuff" as the shuffling mechanism, the protocol is represented by
\begin{equation}
    \calM(D) := \text{shuff}(\{\calA(x_i)| \sigma_i=1, i \in [n]\}), \quad \sigma_i \sim \text{Bern}(\gamma).
    \label{eq:bern}
\end{equation}

  


\subsection{First-cut privacy analysis}
\label{subsec:baseline}
Given that the number of user checking in varies in each round in shuffled check-in, common subsampling lemmas are not applicable.
The straightforward way to tackle this is to analyze, using a high-probability bound, the mechanism with a fixed number of user. 
This approach is taken by \citep{girgis2021shuffled,wu2022walking} which utilize Bernstein inequality and looser shuffling privacy results.
In the following, we give an improved result using a one-sided inequality and utilizing improved shuffling privacy results.

\begin{theorem}[First-cut privacy accounting of shuffled check-in]
\label{th:naive}
Let $n$ be the total number of users, and $\gamma$ be the check-in rate.
Let $B$ be a binomial distribution  parameters $n$ and $\gamma$, and $\beta$ such that it satisfies $Pr(B \geq l)\leq \beta$ for certain $l$ ($l \leq n$).
For $(\epsilon_0, \delta_0)$ LDP randomizer and $\delta$ satisfying $\epsilon_0 \leq \log{\frac{l}{16\log{2/\delta}}}$, the privacy parameters of shuffled check-in are given by
\begin{align}
\epsilon = \log \left( 1 + \left( \frac{e^{\epsilon_0}-1}{e^{\epsilon_0}} \right) \left( \frac{8\sqrt{le^{\epsilon_0} \log{\frac{4}{\delta}}}}{n}+ \frac{8e^{\epsilon_0}}{n} \right) \right),\; \delta = \beta + \delta_l,
\end{align}
where $\delta_l = \frac{l}{n} \left( \delta + (e^{\epsilon}+1)(1+e^{-\epsilon_0}/2)l\delta_0\right)$.
\end{theorem}

\begin{remark}
    The above bound is calculated by ``fixing" the number of user, i.e., $l$, as mentioned above.
    $l$ should be chosen to be small such that $\epsilon$ is small, but it has to be adjusted carefully such that $\delta$ does not exceed any pre-specified value (e.g., the reciprocal of dataset size).
\end{remark}
We are mostly interested in the privacy accounting of the shuffled check-in mechanism when it is applied multiple times (i.e., privacy composition).
To do so, we convert the approximate DP expression to RDP using Equation \ref{eq:dprdp} and perform the privacy accounting and convert it back to approximate DP after accounting for privacy composition.
This method serves as the baseline to be compared with our proposed RDP-based analysis, to be presented next.

\subsection{R{\' e}nyi privacy analysis}
\label{subsec:sci_rdp}
The first-cut approach given previously can only give a conservative estimation of the privacy guarantees as only a fixed number of check-ins is considered.
Our proposed approach using RDP given in the following is more precise in the sense that it considers the contribution of all possible number of check-ins.

\begin{theorem}[RDP of shuffled check-in]
\label{th:main}
Let $D,D'$ be adjacent databases consisting of $n$ clients, and $\gamma$ the check-in rate.
The RDP of order $\lambda$ of the shuffled check-in mechanism is bounded as in Equation \ref{eq:main}.
\begin{equation}
    \eps(\lambda) \leq \frac{1}{\lambda-1} \log \left(
    \sum_{k=0}^{n}  \binom{n}{k} \gamma^{k} (1-\gamma)^{n-k} \mathbb{E}_{q_k}\left[\left(\frac{p_k}{q_k}\right)^{\lambda}\right]
    \right).
    \label{eq:main}
\end{equation}
Here, $p_k$ is the output probability distribution of the subsampled without replacement (of $k$ data instances and sampling rate $k/n$) mechanism.
$q_k$ is the output probability distribution induced by the same mechanism but on $D'$. 
Furthermore, 
Moreover, $\mathbb{E}_{q_k}\left[\left(\frac{p_k}{q_k}\right)^{\lambda}\right]\leq 1 + \sum_{j=2}^{\lambda} \left(\frac{k}{n}\right)^j\binom{\lambda}{j} \zeta^{j}_k(j),$ where $\zeta^{j}_k$ is the ternary DP of the shuffle mechanism with $k$ instances shuffled, can be understood as the RDP of the subsampled (without replacement) shuffle mechanism with subsampling rate $k/n$ and $k$ instances shuffled.
\end{theorem}
\begin{remark}
    In common privacy analyses of subsampling, one separates the subsampling operation from the base mechanism, assuming that the base mechanism is independent of the subsampling operation \citep{balle2018privacy,wang2019subsampled,zhu2019poission}.
Shuffled check-in inherently ``couples" the base mechanism with the subsampling/check-in operation, as shown by terms in the summation of Equation \ref{eq:main}.
\end{remark}

\begin{remark}
Theorem \ref{th:main} states that the shuffled check-in's RDP can be interpreted as the RDP of a subsampled shuffle version of the base mechanism weighted by a binomial distribution.
This expression holds for any base mechanism (Gaussian mechanism, Laplace mechanism, etc.).
\end{remark}

\subsection{Shuffled check-in with approximate LDP}
\label{subsec:sci_adp}
Theorem \ref{th:main} relates the check-in operation to the (shuffled) base mechanism's ternary DP (see Appendix \ref{app:pre} for definition).
In this subsection, we study how to calculate these values concretely.

Given an $(\eps_0,\delta_0)$-LDP randomizer, our strategy is to first convert it to the corresponding RDP parameters using Lemma \ref{lm:dprdp}.
The RDP parameters are then plugged into Equation \ref{eq:main} to calculate the resulting $\eps(\lambda)$.
Note that the subsampling and shuffling mechanisms' approximate DP properties are relatively well studied \citep{feldman2022hiding,wang2019subsampled}.
By converting them into RDP and utilizing Theorem \ref{th:main}, one can calculate the RDP of shuffled check-in in a rather straightforward way.
In the following, we propose two approaches of tackling the problem.

\noindent\textbf{Approach 1: Shuffling conversion.}
Our first approach proceeds in three steps: 1) Convert shuffle DP to shuffle RDP. 2) Use shuffle RDP to evaluate subsampled shuffle RDP. 3) Substitute it into Equation \ref{eq:main} to evaluate the composition of RDP.
To calculate the shuffle DP with $(\eps_0,\delta_0)$-LDP randomizer, we use the equations given below Definition \ref{def:ternary}. 

Note that Step 2 is a calculation of $O(\lambda)$ in terms of time complexity \citep{wang2019subsampled}.
Step 3 involves evaluating the summation with respect to $n$ as in Equation \ref{eq:main}, which is of $O(n)$.
One also needs to convert the RDP notion back to approximate DP using Lemma \ref{lm:rdpdp} (see Appendix \ref{app:pre}), which is an $O(\log\lambda)$ operation, as explained below Lemma \ref{lem:RDP_DP}.
Overall, the shuffling conversion approach is of time complexity $O(n\lambda \log\lambda)$.

\noindent\textbf{Approach 2: Subsampled shuffling conversion.}
Our second approach converts the subsampled shuffle mechanism's DP to RDP: 1) Evaluate subsampled shuffle DP. 2) Convert subsampled shuffle DP to subsampled shuffle RDP. 3) Substitute it into Equation \ref{eq:main} to evaluate the composition of RDP.
The approximate DP of subsampled shuffling can be calculated using the result of \citep{feldman2022hiding} for shuffling and Theorem 9 of \citep{balle2018privacy} for subsampling, with both having $O(1)$ time complexity.
Hence, in terms of time complexity with respect to $\lambda$, the overall approach takes $O(n\log \lambda)$.
Table \ref{tab:comp} summarizes the two approaches mentioned above.

\begin{table*}
\small
    \caption{Our approaches to bounding shuffled check-in mechanism with a generic $(\eps_0,\delta_0)$-LDP randomizer.
    The conversion from DP to RDP (using Lemma \ref{lm:dprdp}) can be performed in two different steps leading to different time complexities as shown in the Table.
    Dependence of the time complexity on other factors is suppressed.}
    \label{tab:comp}
\begin{center}

    \begin{tabular}{ccccc}
        \toprule
        
         &Shuffling&  Subsampling& Convert to RDP from & Time complexity \\
                \midrule
      Approach 1& Use \citep{feldman2022hiding} (DP) &Use \citep{wang2019subsampled} (RDP)& \citep{feldman2022hiding}& $O(n\lambda\log \lambda)$ \\
      Approach 2& Use \citep{feldman2022hiding} (DP)& Use \citep{balle2018privacy} (DP)& \citep{feldman2022hiding}+\citep{balle2018privacy}& $O(n\log \lambda)$ \\
        \bottomrule
    \end{tabular}

\end{center}
\end{table*}

There is no a priori which of these two approaches work better.
Nevertheless, either or both of them perform better than the first-cut/baseline approach given in Section \ref{subsec:baseline}, as they bypass the requirement of conservatively ``fixing" the check-in number to evaluate the amplification effects, resulting in tighter privacy guarantees.
Some numerical evaluations are given in Table \ref{tab:adp}.

\begin{wraptable}{r}{6cm}
\small
\caption{Comparisons of privacy amplification between the baseline approach (Section \ref{subsec:baseline}) and the approaches introduced in Section \ref{subsec:sci_adp} (Approach 1: shuffling conversion, Approach 2: subsampled shuffling conversion).
The number of composition is fixed to 100, dataset size $n=10^4$, $\delta=\delta_1=1/n$, $\delta_0=1/\delta^2$ (parameters in Equation \ref{eq:shuffdelta}).}
\label{tab:adp}
\begin{tabular}{@{}lllll@{}}
\toprule
$\epsilon_0$         & $\gamma$  & Baseline & Approach 1     & Approach 2     \\ \midrule
\multirow{2}{*}{2} & $10^{-2}$ & 0.096    & 0.17           & \textbf{0.049} \\
                   & $10^{-3}$ & 0.013    & \textbf{0.012} & 0.024          \\ \midrule
\multirow{2}{*}{8} & $10^{-2}$ & 9.52     & 8.57           & \textbf{8.18}  \\
                   & $10^{-3}$ & 3.86     & 1.78           & \textbf{1.58}  \\ \bottomrule
\end{tabular}
\end{wraptable}


\section{Lower bound for Gaussian Mechanism}
\label{sec:lower}
The results for shuffled check-in given in Section \ref{subsec:sci_adp} are still somewhat unsatisfactory especially when the LDP randomization mechanism is a Gaussian mechansim.
This is because the Gaussian mechanism satisfies multiple values of $(\epsilon,\delta)$;
more precisely, the Gaussian mechanism is parameterized by $\sigma$, and its \textit{privacy profile} \citep{balle2018privacy}, which writes $\delta$ as a function relating these parameters is $\delta(\epsilon) \leq 1.25e^{-\epsilon^2\sigma^2/2}$ (assuming the sensitivity to be 1). 
Recall that in Section \ref{subsec:sci_adp}, we fixed a particular set of $(\epsilon_0, \delta_0)$ to calculate the privacy amplification of shuffling.
Shuffling introduces another parameter $\delta$, and with composition there is a wide selection of these parameters; picking a specific one may not give the full picture of the mechanism. 
Hence, the results may not be optimal.

The aim of this section is to address this shortcoming by formulating a lower bound to characterize the tightness of the bounds derived in the previous section.
This is done by characterizing the Gaussian mechanism with its parameter, $\sigma$ (and under certain assumptions on the adjacent databases), instead of the privacy parameters.


The lower bound for the R{\'e}nyi divergence (of order $\lambda$) of the Gaussian mechanism in the shuffle model is given below.
\begin{theorem}[Shuffle Gaussian RDP lower bound]
\label{th:shuffgauss}
Consider adjacent databases $D,D' \in \mathbb{R}^n$ with one-dimensional data instances $D = (0, \dotsc, 0)$, $D'= (1,0, \dotsc, 0)$. 
The R{\'e}nyi divergence (of order $\lambda$) of the shuffle Gaussian mechanism with variance $\sigma^2$ for neighboring datasets with $n$ instances is
\begin{align}
     D_{\lambda}(\calM(D)||\calM(D')) = 
    \frac{1}{\lambda-1}\log\left(\frac{e^{-\lambda/2\sigma^2}}{n^\lambda}\sum_{\substack{k_1+\dotsc+k_n=\lambda;\\k_1,\dotsc,k_n\geq 0}}\binom{\lambda}{k_1,\dotsc,k_n}e^{\sum_{i=1}^nk_i^2/2\sigma^2}\right)\label{eq:shuffgauss}
\end{align}
\end{theorem}
\begin{remark}
We prove the lower bound for Gaussian mechanism by studying a specific pair of neighboring datasets perturbed with Gaussian noises.
As the RDP privacy bound is defined by taking the supremum over all neighboring datasets, the lower bound is implied by the fact that the RDP privacy bound for the Gaussian mechanism is at least as the bound given in Equation \ref{eq:shuffgauss}.
\end{remark}
\begin{corollary}[``Upper bound" of Equation \ref{eq:shuffgauss}]
\label{co:upper}
Equation \ref{eq:shuffgauss} is upper-bounded by $\lambda/2\sigma^2$.
\end{corollary}
This corollary states that the upper bound of Equation \ref{eq:shuffgauss} (not to be confused as the upper bound given in Section \ref{sec:sci}) is equal to the Gaussian RDP without shuffling \citep{mironov2017renyi}.
It is interesting to note that this is a strictly tight upper bound compared with previous studies.
In earlier analyses of shuffling \citep{erlingsson2019amplification,balle2020privacy}, privacy amplification occurs at limited parameter region.
We have instead proved in the above corollary, at least for the lower bound, that its value at its largest is analytically and strictly the same as the one without shuffling, and our result is valid for all parameter region.

\noindent\textbf{Numerical computation of the lower bound.}
Evaluating the shuffle Gaussian RDP numerically is non-trivial, as $n$ can be very large ($\gtrsim 10^5$) in distributed learning scenarios and the order of RDP, $\lambda$ to be evaluated has to be large especially when $\delta$ is required to be small.
Calculating the multinomial coefficients can be prohibitively expensive as a result.
In Appendix \ref{app:algo}, we describe our numerical recipes for calculating the RDP, as well as its numerical comparison with respect to the upper bound.

\section{Discussions and Experiments}
\label{sec:discussion}
In this section, we first consider a practical aspect of shuffled check-in in real-life deployment, i.e., modelling dropouts.
Then, we discuss the implication of the techniques developed earlier on applications beyond shuffled check-in (distributed check-in and DP-SGD).
Finally, we present details of the shuffled check-in experiments advertised in Section \ref{sec:introduction}.
In Appendix \ref{app:exp}, we also discuss alternative approaches to studying shuffled check-in, and argue that our approach works the best in practice.

\noindent\textbf{Modelling dropouts.}
In our protocol, each client carries a $p$-biased coin such that she would decide to participate in the training only when head is returned.
Even after a client decides to participate, she may \textit{drop out} due to various reasons, such as battery outage and network disconnection.
Assuming that clients have a certain constant and independent probability of dropping out, $p'$, the \textit{effective} check-in rate is then $\gamma = p(1-p')$.
A practitioner can choose to determine the empirical $\gamma$ (by monitoring the number of client checking in each round), or conservatively use $p$ as the check-in rate to yield a valid privacy upper bound (as smaller check-in rate leads to larger amplification).
These henceforth allow for robust modelling of dropouts. 

Note that assuming a constant $p'$ also means that the analyzer is using a ``flat" prior, i.e., she does not have prior knowledge about the availability of the clients.
A relaxation of this assumption is to assume instead that each user has constant but different effective check-in rates (user may also have the freedom to adjust her check-in rate). 
In this case, one can generalize Equation \ref{eq:main} to accommodate such changes in check-in rate for different users, albeit the evaluation becomes much more complex.

\noindent\textbf{Distributed check-in.}
Our techniques introduced in Section \ref{sec:sci} can also be applied to secure aggregation/distributed DP protocols that do not explicitly sample users. \footnote{This removes the reliance on the server to sample the users explicitly, which in turn reduces the trust requirement on server, as emphasized in Introduction.
Note that recently it is shown that distributed DP protocols can be vulnerable to privacy leakage attacks when the server is fully trusted for orchestrating the whole protocol \citep{boenisch2023federated}.}
Here, each user checks in with certain probability, and sends the encrypted (randomized) report to the server such that only the aggregated (randomized) result is exposed to the server, as in secure aggregation protocols.
We call this class of protocols distributed check-in.
Let us describe the scenario in consideration in detail next.

As in secure aggregation protocols, we consider each client randomizing and encrypting her data instance $x_i$ of dimension $d$ before sending it to the server.
We specifically consider the addition of isotropic Gaussian noise of variance $\sigma^2$: $\tilde{x}_i=x_i+\mathcal{N}(0,\sigma^2I_d)$. \footnote{Alternatively, one can perform randomization using discrete Gaussian \citep{kairouz2021distributed} or Skellam mechanism \citep{agarwal2021skellam,DBLP:journals/pvldb/BaoZXYOTA22}.
These mechanisms have been shown to achieve similar utility as the vanilla Gaussian mechanism.}
Here, w.l.o.g., the global sensitivity of $x_i$ is assumed to be 1.
Note also that the noise added to the aggregated output depends on the number of check-in, which is a variable.
As a consequence, Poisson subsampling privacy analyses available in the literature (such as \citep{zhu2019poission} which applies only to the case where a fixed amount of noise is added to the output) cannot be directly applied to this scenario, and we do not realize any previous work that takes this into account.

The distributed check-in Gaussian RDP can be calculated rather straightforwardly using the tools introduced in Section \ref{sec:sci}. 

\begin{theorem}[Distributed check-in Gaussian RDP]
\label{th:ssg}
The RDP of distributed check-in Gaussian mechanism is
$\epsilon(\lambda) \leq \frac{1}{\lambda-1} \log \left(\sum_{k=1}^{n} \binom{n}{k} \gamma^{k} (1-\gamma)^{n-k} e^{(\lambda-1)\epsilon^{\rm SG}_{k}}\right)$, where $\epsilon^{\rm SG}$ is
    $\eps^{\rm SG}_{m}(\lambda) \leq \frac{1}{\lambda -1}
\log(1 + (\frac{m}{n})^2 \binom{\lambda}{2} \min\left\{4(e^{4/(m\sigma^2)}-1), 2 e^{4/(m\sigma^2)}\right\} +\sum_{j=3}^{\lambda}2(m/n)^j \binom{\lambda}{j} e^{2j(j-1)/(m\sigma^2)})$.
\end{theorem}
Due to space constraints, we relegate the comparison of distributed check-in with subsampling without replacement to Appendix \ref{app:exp}, where we show that distributed check-in can be more private than subsampling without replacement.

\noindent\textbf{DP-SGD.}
The centralized learning paradigm has a discrepancy in privacy accounting of DP-SGD: as noted in \citep{de2022unlocking}, open-source implementations of DP-SGD typically take mini-batches (of fixed size) of samples in a shuffled scheme, but the privacy accountant implemented often ignores the shuffling effect.
One way to resolve this is substituting the central-DP version of DP-SGD with the shuffle model.
Using the shuffle model, a weak trust model, inevitably leads to a drop in utility compared to the highly trusted central-DP model.
To estimate optimistically the drop in utility, we can utilize the lower bound of shuffle Gaussian as given in Theorem \ref{th:shuffgauss}.  

Formally, we consider the mechanism where a fixed number of data instances, $m$, is sampled randomly from $n$ data instances to participate in training.
Each data instance is randomized with Gaussian noises before being shuffled.
We call this overall procedure the subsampled (without replacement) shuffle \citep{girgis2021shuffled} Gaussian mechanism.
The RDP of the subsampled shuffle Gaussian mechanism can be calculated rather straightforwardly combining Equation \ref{eq:shuffgauss} with the generic subsampling theorem derived in \citep{wang2019subsampled}.
It is given by the following.
\begin{theorem}[Subsampled Shuffle Gaussian RDP]
\label{th:subshuff}
Let $n$ be the total number of users, $m$ be the number of subsampled users in each round of training, and $\gamma$ be the subsampling rate, $\gamma = m/n$. Also let $\eps^{\rm SG}_m(\lambda)$ be the shuffle Gaussian RDP given in Theorem \ref{th:shuffgauss}.
The subsampled shuffle Gaussian mechanism satisfies $(\lambda,\eps^{\rm SSG}_{\gamma,m}(\lambda))$-RDP, where 
\begin{align*}
    \eps^{\rm SSG}_{\gamma,m}(\lambda) &\leq \frac{1}{\lambda -1}
\log(1 + \gamma^2 \binom{\lambda}{2} \min\left\{4(e^{\eps^{\rm SG}_m(2)}-1), 2 e^{\eps^{\rm SG}_m(2)}\right\} +\sum_{j=3}^{\lambda}2\gamma^j \binom{\lambda}{j} e^{(j-1)\eps^{\rm SG}_m(j)}
)
\end{align*}
\end{theorem}
We relegate the numerical comparison with moments accountant to Appendix \ref{app:exp}.

\noindent \textbf{Experiments.}
We consider the task of learning under the distributed setting, where there are $n$ clients each holding a data record $x_i$ for $i \in [n]$ (this suffices for our purposes; \textit{user} DP is guaranteed instead of \textit{sample-level} DP when each user holds more than a data record).
The purpose of the system is to train a model with parameter $\theta \in \Theta$ by minimizing a certain loss function $l:\calD^n\times\Theta \to \mathbb{R}_+$ via stochastic gradient descent (SGD), while providing clients with formal privacy guarantees. 

At the beginning of communication round $t$, each client downloads the model $\theta_t$ and flips a biased coin to decide whether to participate in the computation. 
Clients decided to participate calculate the gradient, apply local randomizer to it (e.g., clip the update norm and add Gaussian noise), and send it (encrypted) to the trusted shuffler.
The shuffler permutes randomly all the received messages before forwarding them to the untrusted aggregator.
Then, the aggregator updates the model parameters to $\theta_{t+1}$ using the shuffled messages.
We leave the details of the setup and the full result to Appendix \ref{app:exp}.

We adapt DP-SGD \citep{abadi2016deep,bassily2014private,song2013stochastic} to distributed learning.
Here, the Gaussian mechanism is applied to each client's clipped gradient to achieve $(\eps_0,\delta_0)$-LDP.
More precisely, each sample's gradient (of dimension $d$) is clipped, $g\leftarrow g \cdot{\rm min}\left\{1, \frac{C}{|g|_2}\right\}$, and perturbed with Gaussian noise, $\tilde{g}\to g + \mathcal{N}(0,C^2\sigma^2I_d)$, where $C$ and $\sigma$ are the clipping size and noise multiplier respectively.
The aggregated gradients are then used to update the model at each iteration: $\sum_{i=1}^{B}\left\{\tilde{g}_i +  \mathcal{N}(0, C^2\sigma^2I_d)\right\}= \sum_{i=1}^{B}\tilde{g}_i +  \mathcal{N}(0, BC^2\sigma^2I_d)$, with $B$ the batch size.
Let us remark that, despite its wide usage (under the central DP setting) and simplicity, we do not realize any existing work using the Gaussian mechanism in the LDP/shuffle model literature, possibly due to the difficulty of evaluating an $(\eps_0,\delta_0)$-LDP randomizer.
Hence, we believe that our evaluation is of independent interest as well. 


In Figure \ref{fig:exp}, we plot the accuracy of the models on test dataset with respect to the spent privacy, $\epsilon$, for the MNIST dataset \citep{lecun1998gradient}.
We set the check-in rate $\gamma = 10^{-3}$, $\sigma =0.86$.
100 epochs are run, and the hyperparameters are (optimized via grid search): clipping size $C = 0.05$, learning rate 0.1.
The final $\epsilon$ is obtained by fixing $\delta$ to be the reciprocal of the dataset size.
It is shown that our approaches obtain significantly tighter privacy accounting over the baseline approach (Section \ref{subsec:baseline}).
In Appendix \ref{app:exp}, we also present the lower bounds and the results of training with CIFAR10 \citep{krizhevsky2009learning}.

\section{Conclusion}
This paper attempts to give a privacy analysis of distributed learning under a more security/trust-friendly protocol by taking client ``autonomy" into account, where clients ``check in" without interference of the server.
The trust on the server is further reduced by utilizing the shuffler, an intermediate trust model.
Under this model, we have presented upper and lower bounds characterizing the privacy guarantees of such a protocol, which not only improves over previous results, but also finds implication in other applications (specifically secure aggregation and DP-SGD).

There remain some open problems:
Privacy analysis-wise, although we have obtained a lower bound on the shuffle Gaussian mechanism by considering a specific adjacent datasets, it is still unknown how to acquire the worst-case (taking the supremum over the adjacent datasets) bound tailored to the Gaussian mechanism. 
Protocol-wise, we still rely on the server to initiate FL, determine the number of communication rounds, etc..
A more decentralized approach would be letting the clients collaboratively decide and launch these tasks via, e.g., a distributed ledger \citep{kairouz2021advances}.
We hope that this work can spur the study of privacy amplification within the research community, focusing on minimizing trust assumptions as a step towards practical and deployable distributed protocols. 



\bibliographystyle{plain}
\bibliography{ref}

\clearpage
\appendix
\section{More on Preliminaries}
\label{app:pre}
In this section, we provide additional preliminaries to complement our discussion in Section \ref{subsec:related} and \ref{sec:pre}.

\textbf{DP realizations in FL.}
The popular protocol considered in the literature for achieving central DP is \textit{secure aggregation} \citep{bonawitz2017practical,bell2020secure}. \footnote{This protocol is also said to achieve distributed DP \citep{kairouz2021advances}.}
The implementation proposed in \citep{bonawitz2017practical} has communication complexity of $O(n^2)$ with each user performing $O(n^2)$ computation, where $n$ is the total number of user.
In contrast, the shuffler implementation (alternative implementation include the use of mix-nets \citep{chaum1981untraceable} and peer-to-peer protocols \citep{liew2022network}) using trusted hardware presented in \citep{bittau2017prochlo} requires user to compute only once, and the communication complexity scales linearly with $n$.
Theoretically, the shuffle model lies strictly between the central and local models of DP \citep{cheu2019distributed}.
Moreover, the shuffler is only required to perform the relatively less sophisticated task of masking the data origin (shuffling can even be performed on encrypted data by, e.g., removing only metadata or identifiers, such that sensitive contents are never exposed to the third party/shuffler).
In theory and practice, the shuffle model is a weaker trust model compared to central DP.

\noindent\textbf{Clients' autonomy.}
\citep{balle2020privacy} has considered the possibility of client making independent decision to participate in distributed learning.
However, the scheme relies on a highly trusted ochestrating server (to store the user indices securely), in contrast to our intermediate-trust shuffle model.

\begin{lemma}[Adaptive composition of RDP \citep{mironov2017renyi}]
\label{lm:rdpcompose}
Given two mechanisms $\calM_1,\calM_2$ taking $D\in\calD$ as input that are $(\lambda,\eps_1)$,$(\lambda,\eps_2)$-RDP respectively, the adaptive composition of $\calM_1$ and $\calM_2$ satisfies $(\lambda,\eps_1+\eps_2)$-RDP.
\end{lemma}
While the privacy accounting involving composition is preferably done in terms of RDP, we often need to convert the RDP notion back to the approximate DP notion in the final step. This is achieved with the following. 
\begin{lemma}[RDP-to-DP conversion~\citep{canonne2020discrete,balle2020hypothesis}]\label{lem:RDP_DP} 
\label{lm:rdpdp}
If a mechanism $\calM$ is $\left(\lambda,\epsilon\left(\lambda\right)\right)$-RDP, $\calM$ is also $\left(\epsilon,\delta\right)$-DP, where $1\geq\delta \geq0$ is arbitrary and $\eps$ is given by
\begin{equation}
\label{eq:rdpdp}
\epsilon = \min_{\lambda} \(\epsilon\left(\lambda\right)+\frac{\log\left(1/\delta\right)+\left(\lambda-1\right)\log\left(1-1/\lambda\right)-\log\left(\lambda\right)}{\lambda-1}\).
\end{equation}
\end{lemma}
Typically, we solve for $\epsilon$ given $\delta$ (set to be cryptographically small; we set it to be the reciprocal of dataset size throughout the paper).
To obtain the optimal $\epsilon$, we use a bisection algorithm which takes time $O(\log(\lambda))$ \citep{wang2019subsampled}.


We next introduce $\zeta$-ternary-$|\chi|^{\lambda}$-DP, which is useful when discussing subsampling. 
\begin{definition}[$\zeta$-ternary-$|\chi|^{\lambda}$-DP \citep{wang2019subsampled}]
\label{def:ternary}
A randomized mechanism $\calM:\calX^n\to\calY$ is said to have $\zeta$-ternary-$|\chi|^{\lambda}$-DP, if for any triple of mutually adjacent datasets $\calD,\calD',\calD''\in\calX^n$ ({\em i.e.}, they mutually differ in the same location), the ternary-$|\chi|^{\lambda}$ divergence of $\calM(\calD),\calM(\calD'),\calM(\calD')$ is upper-bounded by $\zeta^{\lambda}(\lambda)$ for all $\lambda\geq 1$ (where $\zeta$ is a function from $\bbR^+$ to $\bbR^+$), {\em i.e.},
\begin{align*}
&D_{|\chi|^{\lambda}}\left(\calM(\calD),\calM(\calD')||\calM(\calD'')\right)  
 \\ 
&:= \mathbb{E}_{\calM(\calD'')}\left[\left|\frac{\calM(\calD)-\calM(\calD')}{\calM(\calD'')}\right|^{\lambda}\right]
\leq \zeta^{\lambda}(\lambda).
\end{align*}
\end{definition}
Moreover, $\zeta(\lambda)$ can be related back to RDP as follows \citep{wang2019subsampled}:
\begin{align*}
\zeta(2)  &\leq 4(e^{\epsilon(2)} - 1),\\
\zeta(j) &\leq e^{(j-1)\epsilon(j)}\text{min}\{2,(e^{\epsilon(\infty)-1})^j\}\;\text{for }j>2
\end{align*}

\noindent\textbf{Privacy amplification via shuffling.}
We use the results of \citep{feldman2022hiding} to calculate privacy amplification due to shuffling.
More precisely, given $n$ reports with each processed with any $(\epsilon_0, \delta_0)$-LDP randomizer and undergone shuffling, and $\epsilon_0 \leq \log(\frac{n}{16 \log(2/\delta_1)})$, the privacy amplification is (see Theorem 3.8 of \citep{feldman2022hiding} for the exact statement) 
\begin{align}
\epsilon &= \log \left( 1 + \left( \frac{e^{\epsilon_0}-1}{e^{\epsilon_0}+1} \right) \left( \frac{8\sqrt{e^{\epsilon_0} \log{\frac{4}{\delta}}}}{\sqrt{n}}+ \frac{8e^{\epsilon_0}}{n} \right) \right),\\
\delta &= \delta_1 + (e^{\epsilon_0}+1)(1+e^{-\epsilon_0}/2)n\delta_0
\label{eq:shuffdelta}
\end{align}

\section{Proofs}
\label{app:proof}
\subsection{Proof of Theorem \ref{th:naive}}
The shuffled check-in mechanism as defined in Equation \ref{eq:bern} has its total number of instances following the binomial distribution with parameters $n$ and $\gamma$.
At each round, with probability $1-\beta$, the total number of check-ins is fewer than $l$.
Note that the subsampling (without replacement) of $l$ instances from $n$ of a $(\epsilon',\delta')$-DP mechanism leads to privacy amplification of the form \citep{balle2018privacy}
\begin{align}
    \epsilon'' = \log(1 + \frac{l}{n}(e^{\epsilon'}-1))
    ,\; \delta''= \frac{l}{n} \delta', 
\end{align}
while the shuffling of $l$ $(\epsilon_0,\delta_0)$-LDP randomizers leads to privacy amplification of the form
\begin{align}
\epsilon' &= \log \left( 1 + \left( \frac{e^{\epsilon_0}-1}{e^{\epsilon_0}+1} \right) \left( \frac{8\sqrt{e^{\epsilon_0} \log{\frac{4}{\delta}}}}{\sqrt{l}}+ \frac{8e^{\epsilon_0}}{l} \right) \right),\\
\delta' &= \delta + (e^{\epsilon_0}+1)(1+e^{-\epsilon_0}/2)l\delta_0
\end{align}
Combining the above results, the subsampled shuffling privacy amplification for $l$ instances can be parameterized as
\begin{align}
    \epsilon_l &= \log \left( 1 + \left( \frac{e^{\epsilon_0}-1}{e^{\epsilon_0}} \right) \left( \frac{8\sqrt{le^{\epsilon_0} \log{\frac{4}{\delta}}}}{n}+ \frac{8e^{\epsilon_0}}{n} \right) \right) \label{eq:conserv1} \\
   \delta_l &= \frac{l}{n} \left( \delta + (e^{\epsilon_l}+1)(1+e^{-\epsilon_0}/2)l\delta_0\right) \label{eq:conserv2} 
\end{align}
Note that both equations are monotonically increasing with respect to $l$.
Hence, we can say that with probability $1-\beta$, there are fewer than $l$ check-ins, and the privacy parameters of shuffled check-in satisfy (conservatively) Equations \ref{eq:conserv1} and \ref{eq:conserv2}.

To finish the proof, we convert the high-probability bound on approximate DP to the usual approximate DP.
We first let $A(x_{1:n}) \in Z$ be the shuffled check-in mechanism, $x_{1:n}$ and $x'_{1:n}$ being the neighboring datasets.
Let also $A^U(x_{1:n}) \in Z$ be the same mechanism but conditioned on $U$ samples being subsampled.
Since each sample is sampled with equal probability, conditioned on $U$, This mechanism can be viewed as subsampling without replacement with probability $U/n$.
Let $\beta$ be the probability such that $Pr(U \geq u) \leq \beta$. \footnote{Here, instead of using the generic and two-sided inequality like Bernstein inequality as in previous work, we calculate numerically the one-sided inequality to obtain a tighter bound.} Then,
\begin{align*}
    Pr(A(x_{1:n}) &\in Z) = \sum_{l\in[n]} Pr(A^U(x_{1:n}) \in Z|U=l) Pr(U=l) \\
&\leq  \sum_{l\in[u]} Pr(A^U(x_{1:n}) \in Z|U=l) Pr(U=l) + \beta \\
&\leq \sum_{l\in[u]} (e^{\epsilon_l}Pr(A^U(x'_{1:n}) \in Z|U=l) + \delta_l)Pr(U=l) + \beta \\
&\leq  e^{\epsilon_u}Pr(A^U(x'_{1:n}) \in Z|U=u) +\delta_u + \beta \\
&\leq e^{\epsilon_u}Pr(A(x'_{1:n}) \in Z)+\delta_u + \beta
\end{align*}
This means that when the shuffled check-in mechanism satisfies $(\epsilon,\delta)$-DP with probability $1-\beta$, it also satisfies  $(\epsilon,\delta+\beta)$-DP (with probability 1).
Applying this to Equations \ref{eq:conserv1} and \ref{eq:conserv2}, we obtain the desired results. 

\subsection{Proof of Theorem \ref{th:main}}

Let $p$ and $q$ be the output probability distributions of our check-in mechanism on adjacent datasets $\calD, \calD'$.
Let the index of the differing data instance be $1$.
The key to our analysis is the decomposition of the mechanism output probability distribution to mixtures of probability distribution at multiple levels.
We first note that the random variables governed by the probability distribution may be divided to two types: algorithm outputs and ``hidden" outputs.
The former ones are those to be integrated or summed over explicitly in the RDP expression.
The latter random variables will be explained in detail later.
For our case, the algorithm outputs are the number of check-in (we denote by $k$), and also the content of the data instance, denoted by $\theta$.

We begin by decomposing the distribution's dependency on $k$.
As each instance checks in with a probability of $\gamma$, the marginal distribution over random variables except $k$ is a binomial distribution:
\begin{align*}
    p = p_k \text{ with probability } \binom{n}{k}\gamma^{k}(1-\gamma)^{n-k} \text{ for }k \in [n]. \\
q = q_k \text{ with probability } \binom{n}{k}\gamma^{k}(1-\gamma)^{n-k} \text{ for }k \in [n].
\end{align*}
where $p_k,q_k$ are distributions acted on by the base mechanism (i.e., shuffling) of which the number of check-ins is fixed to be $k$. 

To calculate $ \mathbb{E}_{q}\left[\left(\frac{p}{q}\right)^{\lambda}\right] $, we decompose the expression with respect to $k$ as follows:
\begin{align*}
    \mathbb{E}_{q}\left[\left(\frac{p}{q}\right)^{\lambda}\right] &= \int d\theta\sum_k \binom{n}{k}\gamma^{k}(1-\gamma)^{n-k} \left(\frac{p_k}{q_k}\right)^{\lambda}q_k \\
    &= \sum_k \binom{n}{k}\gamma^{k}(1-\gamma)^{n-k} \mathbb{E}_{q_k}\left[\left(\frac{p_k}{q_k}\right)^{\lambda}\right]
\end{align*}
We have henceforth reduced the problem of calculating $ \mathbb{E}_{q}\left[\left(\frac{p}{q}\right)^{\lambda}\right]$ to calculating $\mathbb{E}_{q_k}\left[\left(\frac{p_k}{q_k}\right)^{\lambda}\right]$, which is our focus next.

Let $J$ be a random subset of size $k$.
Using $J$, $p_k,q_k$ can further be decomposed as
\begin{align*}
p_k &= \mathbb{E}_J \left[p_k(\cdot| J)\right]= \frac{1}{\binom{n}{k}}\sum_{J}p_k(\cdot| J), \\
q_k &=  \mathbb{E}_J \left[q_k(\cdot| J)\right]= \frac{1}{\binom{n}{k}}\sum_{J}q_k(\cdot| J).
\end{align*}
$J$ is a ``hidden" variable that does not appear explicitly in the RDP expression, but is important at achieving amplification similar to amplification via subsampling.
This can be seen by first letting $E$ to be the subset containing the 1st element ($E^c$ being the complement subset).
Also, note that there are $\binom{n-1}{k-1}$ in $E$ and $\binom{n-1}{k}$ components in $E^c$. 
Then, we decompose $p_k$ as
\begin{align*}
p_k = \frac{1}{\binom{n}{k}}\sum_{J}p_k(\cdot| J) &= \frac{k/n}{\binom{n-1}{k-1}}\sum_{J\subset E}p_k(\cdot| J) + \frac{1-k/n}{\binom{n-1}{k}}\sum_{J\subset E^c}p_k(\cdot| J) \\
& = \frac{k}{n} \mathbb{E}_{J\subset E} \left[p_k(\cdot| J)\right] + (1-\frac{k}{n})\mathbb{E}_{J\subset E^c} \left[p_k(\cdot| J)\right]
\end{align*}

To be more succinct, we slightly abuse the notion to write $\mathbb{E}_{J\subset E}\left[ p_k(\cdot| J)\right]$ as $ p_k(\cdot| E)$ and $\mathbb{E}_{J\subset E^c} \left[p_k(\cdot| J)\right]$ as $ p_k(\cdot| E^c)$
Similarly for $q_k$,
$$
q_k = \frac{k}{n} q_k(\cdot| E)+ (1-\frac{k}{n})q_k(\cdot| E^c)
$$
With these, we have decomposed $p_k,q_k$ to components with/without the first data instance.
By construction, $p_k(\cdot| E^c) = q_k(\cdot| E^c)$.
With this construction and by comparing the results given in \citep{wang2019subsampled}, one can see that $\mathbb{E}_{q_k}\left[\left(\frac{p_k}{q_k}\right)^{\lambda}\right]$ is the RDP of the subsampled (without replacement) mechanism with subsampling rate $k/n$.

It is clear that $p_k -q_k = \frac{k}{n} (p_k(\cdot| E) - q_k(\cdot| E))$.
Applying this to the ternary-$|\chi|^j$-divergence, we have
\begin{align*}
D_{|\chi|^{j}}(p_k,q_k||r)  
 = \left(\frac{k}{n}\right)^j D_{|\chi|^{j}}(p(\cdot | E),q(\cdot | E)||r)
\end{align*}

We can then follow the procedure in Proposition 21 of \citep{wang2019subsampled} to show that $D_{|\chi|^{j}}(p_k,q_k||r) \leq \left(\frac{k}{n}\right)^j \zeta^j_k(j)$.


The above calculation is related to what we want to calculate, $ \mathbb{E}_{q_k}\left[\left(\frac{p_k}{q_k}\right)^{\lambda}\right] $, by
\begin{align*}
    \mathbb{E}_{q_k}\left[\left(\frac{p_k}{q_k}\right)^{\lambda}\right]
    & \leq  1 + \sum_{j=2}^{\lambda} \left(\frac{k}{n}\right)^j\binom{\lambda}{j} \zeta^{j}_k(j)
\end{align*}

\noindent \textit{Showing} $D_{|\chi|^{j}}(p_k,q_k||r) = \left(\frac{k}{n}\right)^j D_{|\chi|^{j}}(p(\cdot | E),q(\cdot | E)||r)$.
\begin{align*}
D_{|\chi|^{j}}(p_k,q_k||r) &=\mathbb{E}_r\left[\left(\frac{|p_k-q_k|}{r}\right)^j\right] \\
&= \left(\frac{k}{n}\right)^j \mathbb{E}_r\left[\left(\frac{|p_k(\cdot| E) - q_k(\cdot| E)|}{r}\right)^j\right]\\
&= \left(\frac{k}{n}\right)^j D_{|\chi|^{j}}(p(\cdot | E),q(\cdot | E)||r)
\end{align*}
\noindent \textit{Showing}
\begin{align*}
    \mathbb{E}_{q_k}\left[\left(\frac{p_k}{q_k}\right)^{\lambda}\right]
    & \leq  1 + \sum_{j=2}^{\lambda} \left(\frac{k}{n}\right)^j\binom{\lambda}{j} \zeta^{j}_k(j)
\end{align*}
Note that
\begin{align*}
    \mathbb{E}_{q_k}\left[\left(\frac{p_k}{q_k}\right)^{\lambda}\right] &= 1 + \binom{\lambda}{1}\mathbb{E}_{q}\left[\frac{p_k}{q_k}-1\right] + \sum_{k=2}^{\lambda}\binom{\lambda}{k}\mathbb{E}_{q_k}\left[\left(\frac{p_k}{q_k}-1\right)^k\right] \\
    & \leq 1 + \sum_{j=2}^{\lambda} \left(\frac{k}{n}\right)^j\binom{\lambda}{j} \underset{p_k,q_k,r}{\text{sup}} D_{|\chi|^{j}}(p_k,q_k||r) \\
    &= 1 + \sum_{j=2}^{\lambda} \left(\frac{k}{n}\right)^j\binom{\lambda}{j} \zeta^{j}_k(j)
\end{align*}
using $\mathbb{E}_{q_k}\left[\left(\frac{p_k}{q_k}\right)\right] = 1$ and 
\begin{align*}
    \mathbb{E}_{q_k}\left[\left(\frac{p_k}{q_k}-1\right)^j\right] = \mathbb{E}_{q_k}\left[\left(\frac{p_k-q_k}{q_k}\right)^j\right] & \leq \underset{p_k,q_k,r}{\text{sup}}\mathbb{E}_{r}\left[\left(\frac{p_k-q_k}{r}\right)^j\right] \\
&= \zeta^{j}_k(j)
\end{align*}

\subsection{Proof of Theorem \ref{th:shuffgauss}}
Here, we give the full proof of Theorem \ref{th:shuffgauss}, restated below.
\begin{theorem}[Shuffle Gaussian RDP lower bound restated]
Consider adjacent databases $D,D' \in \mathbb{R}^n$ with one-dimensional data instances $D = (0, \dotsc, 0)$, $D'= (1,0, \dotsc, 0)$. 
The R{\'e}nyi divergence (of order $\lambda$) of the shuffle Gaussian mechanism with variance $\sigma^2$ for neighboring datasets with $n$ instances is given by
\begin{align}
   &  D_{\lambda}(\calM(D)||\calM(D')) = \nonumber \\
   & \frac{1}{\lambda-1}\log\left(\frac{e^{-\lambda/2\sigma^2}}{n^\lambda}\sum_{\substack{k_1+\dotsc+k_n=\lambda;\\k_1,\dotsc,k_n\geq 0}}\binom{\lambda}{k_1,\dotsc,k_n}e^{\sum_{i=1}^nk_i^2/2\sigma^2}\right)
\end{align}
\end{theorem}

\begin{proof}
Note that the databases of interest can be represented in the form of $n$-dimensional vector: where $D$ is simply $0$, an $n$-dimensional vector with all elements equal to zero, and $D' = e_1$, where $e_i$ is a unit vector with the elements in all dimensions except the $n$-th ($i\in[n]$) one equal to zero.

The shuffle Gaussian mechanism first applies Gaussian noise (with variance $\sigma^2$) to each data instance, and subsequently shuffle the noisy instances.
The shuffled output of $D$ distributes as $\mathcal{M}(D) \sim (\mathcal{N}(0,\sigma^2), \dotsc, \mathcal{N}(0,\sigma^2))$, an $n$-tuple of $\mathcal{N}(0,\sigma^2)$.
Intuitively, the adversary sees a anonymized set of randomized data of size $n$. 
Due to the homogeneity of $D$, we can write
$\mathcal{M}(D) \sim \mathcal{N}(0,\sigma^2I_n)$ using the vector notation mentioned above. \footnote{Note that the vector space defined here is over the dataset dimension instead of input dimension.}

On the other hand, since all data instances except one is 0 in $D'$, the mechanism can be written as a mixture distribution: $\mathcal{M}(D')\sim \frac{1}{n} (\mathcal{N}(e_1,\sigma^2I_n)+ \dotsc + \mathcal{N}(e_n,\sigma^2I_n))$.
Intuitively, each output has probability $1/n$ carrying the non-zero element, and hence the mechanism as a whole is a uniform mixture of $n$ distributions.
In summary, the neighbouring databases of interest are
\begin{align}
\mathcal{M}(D) &\sim \mathcal{N}(0,\sigma^2I_n), \nonumber\\
\mathcal{M}(D')&\sim \frac{1}{n} \left(\mathcal{N}(e_1,\sigma^2I_n)+ \dotsc + \mathcal{N}(e_n,\sigma^2I_n)\right).
\label{eq:specialdb}
\end{align}


We are concerned with calculating the following quantity, 
$$\mathbb{E}_{x\sim \mathcal{M}(D)}\left[\left(\frac{\mathcal{M}(D')(x)}{\mathcal{M}(D)(x)}\right)^\lambda\right]$$
for RDP:
\begin{align}
    &\mathbb{E}_{x\sim \mathcal{M(D)}}\left[\left(\frac{\mathcal{M}(D')(x)}{\mathcal{M}(D)(x)}\right)^\lambda\right]\nonumber \\ 
    &= \int \left(\frac{\exp{\left[-(x_1-1)^2/2\sigma^2-\sum_{i=2}^n x_i^2/2\sigma^2\right]}+\dotsc}{n\exp[-\sum_{i=1}^nx_i^2/2\sigma^2]}\right)^{
\lambda} \nonumber \\
&\exp[-\sum_{i=1}^nx_i^2/2\sigma^2] \frac{d^nx}{(2\pi \sigma^2)^{n/2}} \label{eq:shuffgauss1} \\
&= \int \left(\frac{\exp{\left[-(x_1-1)^2/2\sigma^2+x_1^2/2\sigma^2\right]}+\dotsc}{n}\right)^{
\lambda}\nonumber \\
&\exp[-\sum_{i=1}^nx_i^2/2\sigma^2] \frac{d^nx}{(2\pi \sigma^2)^{n/2}} \label{eq:shuffgauss2}\\
&= \int \left(\exp{\left[(2x_1-1)/2\sigma^2\right]} + \dotsc \right)^{\lambda}\exp[-\sum_{i=1}^nx_i^2/2\sigma^2]\frac{d^nx}{n^\lambda(2\pi \sigma^2)^{n/2}} \label{eq:main3} \\
&= \int \left(\sum_{i=1}^n\exp{\left[(2x_i-1)/2\sigma^2\right]}\right)^{\lambda}\exp[-\sum_{i=1}^nx_i^2/2\sigma^2]\frac{d^nx}{n^\lambda(2\pi \sigma^2)^{n/2}}.\label{eq:shuffgauss4}
\end{align}
Let us explain the above calculation in detail.
We first notice that $\mathbb{E}_{x\sim \mathcal{M(D)}}\left[\left(\frac{\mathcal{M}(D')(x)}{\mathcal{M}(D)(x)}\right)^\lambda\right]$ is an $n$-th dimensional integral of $x_i$ ($i\in[n]$), as shown in Equation \ref{eq:shuffgauss1}.
For each term $j\in[n]$ in the nominator of $(\dotsc)^\lambda$, we divide it by the denominator $\exp[-\sum_{i=1}^nx_i^2/2\sigma^2]$, yielding $\exp[-(x_j-1)^2/2\sigma^2-\sum_{i\neq j}^nx_i^2/2\sigma^2+\sum_{i=1}^nx_i^2/2\sigma^2]= \exp[-(x_j-1)^2/2\sigma^2 + x_j^2/2\sigma^2]$, with all $x_i$ terms in $i\in[n]$ except $j$ canceled out (Equation \ref{eq:shuffgauss2}).
The term can be further simplified to $\exp[(2x_j-1)/2\sigma^2]$ as in Equation \ref{eq:shuffgauss4}.

Then, we expand the expression $(\dotsc)^\lambda$ using the multinomial theorem:
\begin{align}
\label{eq:k-integral_}
    \left(\sum_{i=1}^n\exp{\left[(2x_i-1)/2\sigma^2\right]}\right)^{\lambda} &=\sum_{\substack{k_1+...+k_n=\lambda;\\k_1,\dotsc,k_n\geq 0}} \binom{\lambda}{k_1,\dotsc,k_n}\nonumber \\
   & \prod_{i=1}^n\exp\left[k_i(2x_i-1)/2\sigma^2\right],
\end{align}
where $\binom{\lambda}{k_1,\dotsc,k_n} = \frac{\lambda!}{k_1!k_2!\dotsc k_n!}$ is the multinomial coefficient, and $k_i \in \mathbb{Z^+}$ for $i\in[n]$.

Before proceeding, we make a detour to calculate the following integral (with $k\in \mathbb{Z^+}$):
\begin{align}
    &\int \exp[k(2x-1)/2\sigma^2]\exp[-x^2/2\sigma^2]\frac{dx}{\sqrt{2\pi\sigma^2}} \nonumber \\
    &= \int \exp[-(x-k)^2/2\sigma^2 +(k^2-k)/2\sigma^2]\frac{dx}{\sqrt{2\pi\sigma^2}} \nonumber \\
    & = \exp[(k^2-k)/2\sigma^2]. \nonumber
\end{align}

Using the above expression and Equation \ref{eq:k-integral_}, we can write Equation \ref{eq:shuffgauss4} as
\begin{align}
&\int \left(\sum_{i=1}^n\exp{\left[(2x_i-1)/2\sigma^2\right]}\right)^{\lambda}\exp[-\sum_{i=1}^nx_i^2/2\sigma^2]\frac{d^nx}{n^\lambda(2\pi \sigma^2)^{n/2}} \nonumber\\
    &=\int\sum_{\substack{k_1+...+k_n=\lambda;\\k_1,\dotsc,k_n\geq 0}} \binom{\lambda}{k_1,\dotsc,k_n}\prod_{i=1}^n\exp\left[k_i(2x_i-1)/2\sigma^2\right] \frac{\exp[-x_i^2/2\sigma^2]d^nx}{n^\lambda(2\pi \sigma^2)^{n/2}} \nonumber \\
    &= \frac{1}{n^\lambda}\sum_{\substack{k_1+...+k_n=\lambda;\\k_1,\dotsc,k_n\geq 0}} \binom{\lambda}{k_1,\dotsc,k_n}\prod_{i=1}^n\exp[(k_i^2-k_i)/2\sigma^2],
\end{align}
where we have moved the $1/n^\lambda$ factor to the front.
Noticing that $\prod_{i=1}^n\exp[-k_i] = \exp[-\sum_i^{n}k_i]$ and that $\sum_{i=1}^nk_i=\lambda$ under the multinomial constraint, we have
\begin{align}
& \frac{1}{n^\lambda}\sum_{\substack{k_1+...+k_n=\lambda;\\k_1,\dotsc,k_n\geq 0}} \binom{\lambda}{k_1,\dotsc,k_n}\prod_{i=1}^n\exp[(k_i^2-k_i)/2\sigma^2] \nonumber\\
    &=\frac{e^{-\lambda/2\sigma^2}}{n^\lambda}\sum_{\substack{k_1+\dotsc+k_n=\lambda;\\k_1,\dotsc,k_n\geq 0}}\binom{\lambda}{k_1,\dotsc,k_n}e^{\sum_{i=1}^nk_i^2/2\sigma^2}. \label{eq:shuffgauss5}
\end{align}
We henceforth obtain the expression given in Equation \ref{eq:shuffgauss} as desired.
\end{proof}

\subsection{Proof of Corollary \ref{co:upper}}
\begin{proof}
Applying the Cauchy-Schwartz inequality, $\sum_{i=1}^n k_i^2 \leq (\sum_{i=1}^n k_i)^2 = \lambda^2$, to the term inside the logarithm of Equation \ref{eq:shuffgauss},
\begin{align*}
   & \frac{e^{-\lambda/2\sigma^2}}{n^\lambda}\sum_{\substack{k_1+\dotsc+k_n=\lambda;\\k_1,\dotsc,k_n\geq 0}}\binom{\lambda}{k_1,\dotsc,k_n}e^{\sum_{i=1}^nk_i^2/2\sigma^2} \\
   &\leq \frac{e^{-\lambda/2\sigma^2}}{n^\lambda}\sum_{\substack{k_1+\dotsc+k_n=\lambda;\\k_1,\dotsc,k_n\geq 0}}\binom{\lambda}{k_1,\dotsc,k_n}e^{\lambda^2/2\sigma^2}\\
    &= e^{(\lambda^2-\lambda)/2\sigma^2},
\end{align*}
noting that  $\sum_{\substack{k_1+\dotsc+k_n=\lambda;\\k_1,\dotsc,k_n\geq 0}}\binom{\lambda}{k_1,\dotsc,k_n}=n^\lambda$, we get the desired expression.
\end{proof}
\section{Algorithms}
\label{app:algo}
\subsection{Details of the protocol}
The full shuffled check-in protocol is described in Algorithm \ref{alg:protocol}.

\begin{figure}
\begin{small}
\centering
\begin{minipage}[t]{0.52\textwidth} 
\begin{center}
\begin{algorithmic}
\SUB{Server-side protocol:}
   \STATE \emph{parameters:}  Number of rounds $T$
   \STATE Initialize model $\theta_1$
   \FOR{$t \in [T]$}
        \STATE Broadcast model $\theta_t$ and request users to participate in training
        \STATE Shuffler receives client responses and outputs the shuffled messages $(\tilde{g}_t)$ to aggregator
        \STATE $\theta_{t+1}$ $\leftarrow$ $\text{ModelUpdate}(\theta_t; \tilde{g}_t)$
        \STATE Output $\theta_{t+1}$
   \ENDFOR
\end{algorithmic}
\end{center}
\vfill
\end{minipage}
\hfill
\begin{minipage}[t]{0.47\textwidth}
\begin{center}
\begin{algorithmic}
 \SUB{Client-side protocol for User $i$:}
  \STATE \emph{parameters:} probability $p$, loss function $l$, local randomizer $\calA$, Number of rounds $T$
  \STATE \emph{private data:} $x_i \in \calD$
   \FOR{$t \in [T]$}
  \STATE Receive request to participate in training
  \IF{a $p$-biased coin returns head}
    \STATE Download $\theta_t$
    \STATE $\tilde{g}_{t} \leftarrow \calA(\nabla_\theta  l(\theta_t;x_i))$
    \STATE Send $\tilde{g}_{t}$ to shuffler
  \ENDIF
  \ENDFOR
  
\end{algorithmic}
\end{center}
\vfill
\end{minipage}
\end{small}

\mycaptionof{algorithm}{The distributed protocol of shuffled check-in.}
\label{alg:protocol}

\end{figure}
\subsection{More on the numerical computation of the lower bound}
We exploit the permutation invariance inherent in Equation \ref{eq:shuffgauss} to perform more efficient computation.
Equation \ref{eq:shuffgauss} contains a sum of multinomial coefficients multiplied by an expression with integer $k_i$ ($i\in[n]$) constrained by $k_1 + \dotsc +k_n = \lambda$.
We first obtain all possible $k_i$'s with the above constraint, without counting same summands differing only in the order.
This is integer partition, a subset sum problem known to be NP.
It is also related to partition in number theory \citep{andrews2004integer}.
Nevertheless, the calculation can be performed rather efficiently in practice \citep{kelleher2009generating}.
We denote this operation by \texttt{GetPartition}.


Note that there is a one-to-one mapping between the term $e^{\sum_{i=1}^nk_i^2/2\sigma^2}$ in Equation \ref{eq:shuffgauss5} and the partition obtained above.
We can therefore calculate the number of permutation corresponding to each partition, multiplied by the multinomial coefficient, and $e^{\sum_{i=1}^nk_i^2/2\sigma^2}$, and sum over all the partitions to obtain a numerical expression of Equation \ref{eq:shuffgauss5}.
For a partition with the number of repetition of exponents with the same degrees (or the counts of unique $k_1,\dotsc,k_n$; denote this operation by \texttt{GetUniqueCount}) being $\kappa_1, \dotsc, \kappa_{l}$ ($l\leq n$), the number of permutation can be calculated to be $n!/(\kappa_1!\dotsc \kappa_l!)$.
We describe our algorithm for evaluating Equation \ref{eq:shuffgauss} in Algorithm \ref{alg:num}.

\begin{algorithm}[t]
\caption{Numerical evaluation of Equation \ref{eq:shuffgauss}.}
\label{alg:num}
\begin{algorithmic}[1]
\STATE\textbf{Inputs:} Database size $n$, RDP order $\lambda$, variance of Gaussian mechanism $\sigma^2$.
\STATE \textbf{Output:} RDP parameter $\eps$.
\STATE \textbf{Initialize:} $\Gamma:\{\emptyset\}$, $\texttt{Sum}=0$.
\STATE $\Gamma \leftarrow \texttt{GetPartition}(\lambda)$
\FOR { $\rho \in \Gamma$}
\STATE $\{\kappa_i\} \leftarrow \texttt{GetUniqueCount}(\rho)$
\STATE $\texttt{Sum} \leftarrow \texttt{Sum} + \frac{e^{-\lambda/2\sigma^2}}{n^\lambda}\binom{\lambda}{k_1\dotsc k_n}\frac{n!}{\kappa_1!\dotsc\kappa_i!} e^{\sum_{j=1}^{n}k_j^2}$ where $\{k_i\}\in \rho$
\ENDFOR
\STATE \textbf{Return:} $\eps = \texttt{Sum}$
\end{algorithmic}
\end{algorithm}

\section{More Results}
 \label{app:exp}
\subsection{Alternative approaches}
Here, we comment on several alternative strategies to analyze the shuffled check-in mechanism and argue why they are not suitable.

\noindent \textbf{Privacy profile and coupling.}
In Section \ref{subsec:sci_adp}, we convert the approximate DP of two base mechanisms, shuffle and subsampled shuffle mechanisms, to RDP for privacy accounting. 
One may wonder if it is possible to derive the approximate DP of the shuffle check-in mechanism directly, without using shuffle or subsampled shuffle mechanisms, and finally convert it to RDP.
Here, we show that while it is theoretically possible to derive such an expression, but numerical difficulties arise making its evaluation impractical.

In \citep{balle2018privacy}, a general treatment of subsampling has been giving using the concepts of privacy profile and distribution coupling.
While their approach to Poisson subsampling under replacement DP assumes that the base mechanism decouples from the subsampling operation (Theorem 14 of \citep{balle2018privacy}), it can be extended to our case quite straightforwardly.

Following the notation of \citep{balle2018privacy}, we let $\mathcal{M'}$ denote the shuffled check-in mechanism, and $\mathcal{M}_k$ be the shuffle mechanism on dataset of size $k$.
Then, the privacy profiles of these mechanisms ($\delta_{\mathcal{M'}}$ and $\delta_{\mathcal{M}_k}$) are related as follows.

$$
\delta_{\mathcal{M'}}(\epsilon) \leq \gamma \beta \delta_{\mathcal{M}_0}(\epsilon_0) + \gamma (1-\beta)\left(\sum_{k=1}^{n-1}\tilde{\gamma}_k 
\delta_{\mathcal{M}_k}(\epsilon_k) + \tilde{\gamma}_n\right)
$$
Here, $\epsilon = \log(1+\gamma(e^{\epsilon_0}-1))$, $\beta = e^{\epsilon-\epsilon_0}$, $\epsilon_k = \epsilon_0 + \log(\frac{\gamma}{1-\gamma}(\frac{n}{k}-1))$, and $\tilde{\gamma}_k = \binom{n-1}{k-1}\gamma^{k-1}(1-\gamma)^{n-k}$.

To solve for $\delta_{\mathcal{M'}}(\epsilon)$, we need to calculate the shuffling privacy profile $\delta_{\mathcal{M}_k}$.

Although the numerical routine to calculate the privacy profile of an $\epsilon_0$-LDP randomizer is given in \citep{feldman2022hiding}, we find that one needs to carefully tune the parameter $\delta_1$ as in Equation \ref{eq:shuffdelta} to obtain reasonable values for the privacy profile (between 0 and 1).
Furthermore, the final result is unstable numerically as it fluctuates wildly with the tuning parameter.
The problem worsens for shuffled check-in as such a tuning has to be performed separately for each $k$ as shown in the above equation's sum operation. 
To this end, we resort to solving the privacy accounting problem through the conversion of shuffle/subsampled shuffle mechanism as introduced in Section \ref{subsec:sci_adp}, which is a more stable approach.
 
\noindent \textbf{Numerical accountant.}
It is recently shown that performing even tighter privacy accounting is possible via privacy loss distributions and fast Fourier transforms \citep{koskela2020computing,gopi2021numerical}.
However, due to computational constraints, these techniques are limited to privacy loss distributions that can be reduced to a low-dimensional problem, and therefore not applicable to shuffled check-in, a high-dimensional mixture distribution of $n$ components.
See e.g., Section 5 of \citep{koskela2023numerical}, where a similar problem occurs in the shuffle model.
It is henceforth more appropriate to attack the privacy composition problem of shuffled check-in with RDP.

 \subsection{Experimental details}
 \noindent \textbf{Shuffle Gaussian.}
Figure \ref{fig:shuffgauss} compares the privacy amplification with respect to the number of composition, between the lower bound and the upper bound obtained by converting the shuffle approximate-DP to RDP.
Parameters are set to be the same as in Table \ref{tab:adp}: $\delta=\delta_1=1/n$, $\delta_0=1/\delta^2$. 
Closing the gap between the lower and upper bound is an interesting direction for future work.

\begin{figure}
        \centering 
    \includegraphics[width=0.45\textwidth]{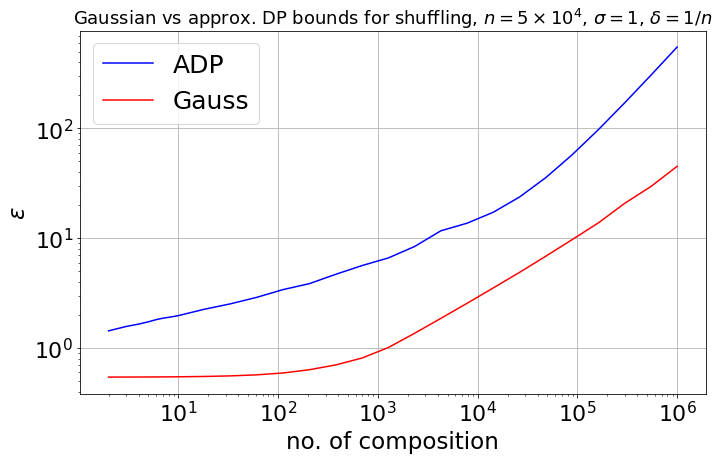}
  \caption{Lower bound (Equation \ref{eq:shuffgauss}, red) is compared with the shuffle approximate-DP (blue) varying the number of composition.}
      \label{fig:shuffgauss}
\end{figure}
 
\noindent \textbf{Distributed Check-in.}
 In Figure \ref{fig:ddp}, we show that distributed check-in can have much suppressed $\epsilon$ compared to the subsampling without replacement case fixing the subsampling rate to be equal to the check-in rate (that is, the subsampling without replacement RDP is $\epsilon_{k}^{\text SG}$ with $k=n\gamma$), and Gaussian noise of $\sigma=1$ added to each individual report for both cases.

\noindent\textbf{DP-SGD.}
Figure \ref{fig:dpsgd} shows the privacy-utility trade-offs between the shuffle and central DP models for training a convolutional neural network with the standard MNIST dataset \citep{lecun2010mnist}.
Here, In the shuffle model, each sample's gradient (of dimension $d$) is clipped, $g\leftarrow g \cdot{\rm min}\left\{1, \frac{C}{|g|_2}\right\}$, and perturbed with Gaussian noise, $\tilde{g}\to g + \mathcal{N}(0,C^2\sigma^2I_d)$, where $C$ and $\sigma$ are the clipping size and noise multiplier respectively.
The aggregated gradients are then used to update the model at each iteration: $\sum_{i=1}^{B}\left\{\tilde{g}_i +  \mathcal{N}(0, C^2\sigma^2I_d)\right\}= \sum_{i=1}^{B}\tilde{g}_i +  \mathcal{N}(0, BC^2\sigma^2I_d)$, with $B$ the batch size.
The central-DP model adds Gaussian noises with noise multiplier $\sigma_{\text{CDP}}$ to the \textit{aggregated} gradients of a batch instead:  
$\sum_{i=1}^{B}\tilde{g}_i + \mathcal{N}(0, C^2\sigma_{\text{CDP}}^2I_d)$.
For fair comparisons, we set $\sigma_{\text{CDP}} = \sqrt{B}\sigma$ such that the total amount of noise added is the same for both models.

\begin{figure}
        \centering 
    \includegraphics[width=0.45\textwidth]{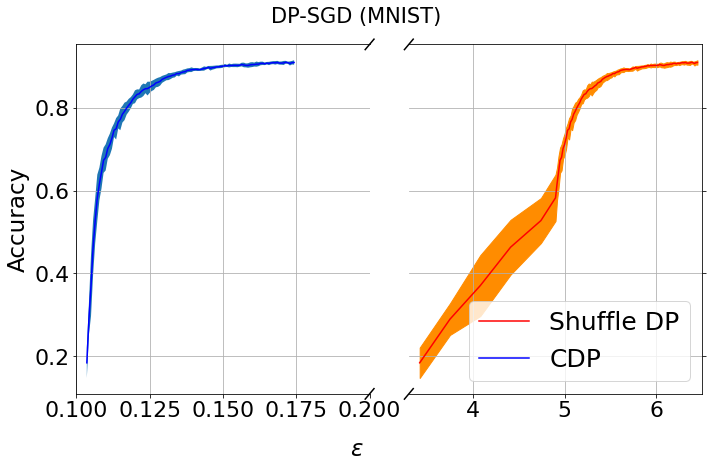}
  \caption{Comparison of privacy-utility trade-offs of the standard DP-SGD, under the central DP (blue) and the shuffle DP (orange) trust models.}
      \label{fig:dpsgd}
\end{figure}

\begin{figure}
        \centering 
    \includegraphics[width=0.45\textwidth]{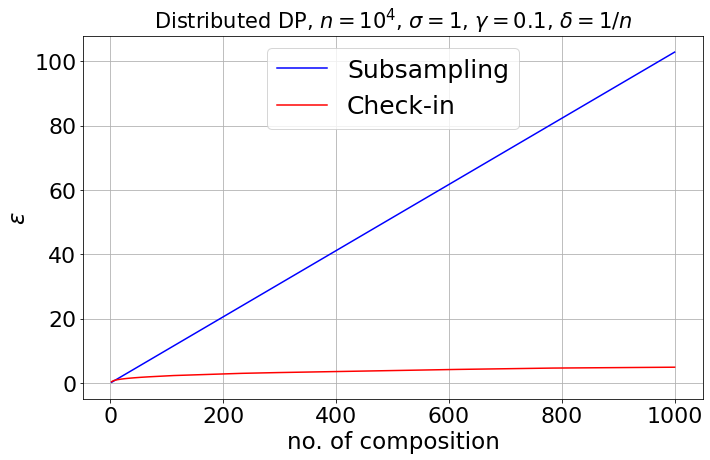}
  \caption{Comparison between privacy amplification via subsampling without replacement (blue), and distributed check-in (red), with respect to the number of composition. 
  The subsampling rate is set equal to the check-in rate, $\gamma=0.1$.}
      \label{fig:ddp}
\end{figure}

\noindent \textbf{Distributed learning.}
The datasets in consideration are MNIST handwritten digit dataset \citep{lecun1998gradient} and CIFAR-10 dataset \citep{krizhevsky2009learning}.
Each user holds a data instance from the train dataset (hence the population size is 60,000 and 50,000 for MNIST and CIFAR-10 respectively), and the trained model's accuracy is evaluated with the test dataset.
For MNIST, we train a convolutional neural network similar to the one used in \citep{erlingsson2020encode} from scratch.
For CIFAR-10, we first pre-train a neural network with CIFAR-100 \citep{krizhevsky2009learning} without adding noise assuming that CIFAR-100 is a public dataset.
Then, the final layer is fine-tuned privately with CIFAR-10.

One fact worth mentioning for the adaptation to the distributed setting is that as we are working with replacement DP, the sensitivity should be set to twice of those given in \citep{abadi2016deep}, which is defined in terms of addition/removal DP \citep{dwork2006calibrating,vadhan2017complexity}.

We show our full results in Figure \ref{fig:exp2}.

\begin{figure}[t]
  \centering
  \begin{subfigure}{0.45\textwidth}
      \centering 
    \includegraphics[scale=0.2]{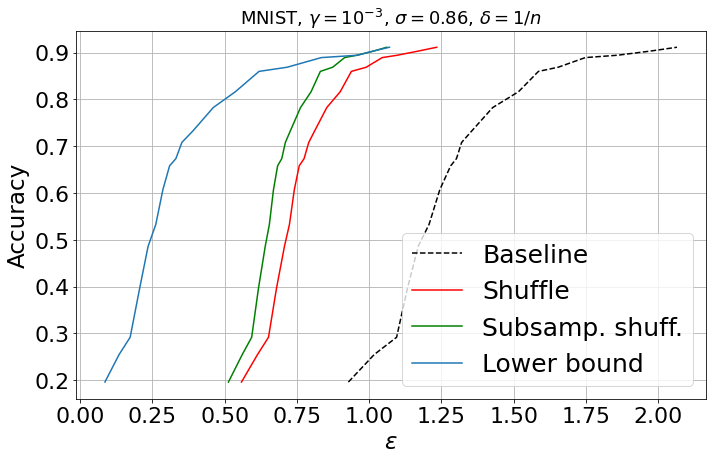}
    \caption{MNIST.}
    \label{fig:mnist}
  \end{subfigure}\hfil 
  \begin{subfigure}{0.45\textwidth}
      \centering 
    \includegraphics[scale=0.2]{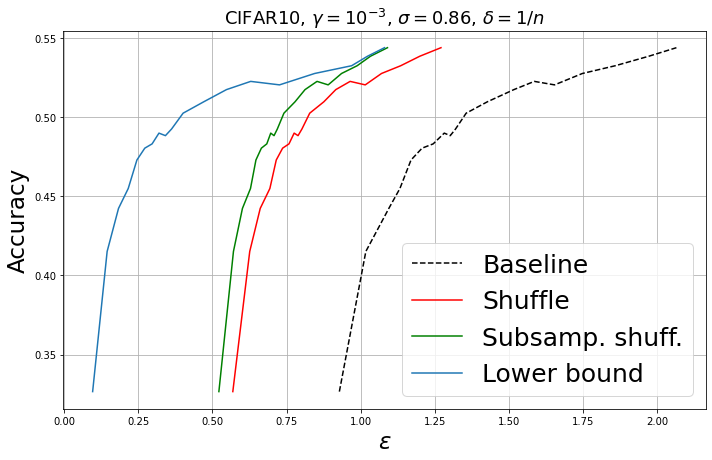}
    \caption{CIFAR.}
    \label{fig:cifar}
  \end{subfigure}\hfil 
  \caption{Upper (red and green) and lower (blue) bounds of the shuffled check-in mechanism applied to machine learning tasks, based on our analysis. These are tighter than the baseline approach (black dotted) based on existing work.}
  \label{fig:exp2}
  \end{figure}
  
\noindent \textbf{Network architecture.}
The neural network we use in the experiments presented in Section \ref{sec:discussion} is as in Table \ref{tab:nn}.
\begin{table}[]
 \caption{Neural network architecture used in our experiment.}
    \label{tab:nn}
    \centering
    \begin{tabular}{c|c}
    \toprule
        Layer & Parameters  \\
                        \midrule
        Convolution & 16 filters of $8\times8$ , strides 2 \\
        Max-pooling & $2\times2$ \\
        Convolution & 32 filters of $4\times4$, strides 2\\
        Max-pooling & $2\times2$ \\
        Linear & 32 units \\
        Softmax & 10 units \\
        \bottomrule
    \end{tabular}
\end{table}

\end{document}